\documentclass[lettersize,journal]{IEEEtran}
\usepackage{amsmath,amsfonts}
\usepackage{algorithmic}
\usepackage{algorithm}
\usepackage{array}
\usepackage[caption=false,font=normalsize,labelfont=sf,textfont=sf]{subfig}
\usepackage{textcomp}
\usepackage{stfloats}
\usepackage{url}
\usepackage{verbatim}
\usepackage{graphicx}
\usepackage{cite}
\usepackage[caption=false,font=normalsize,labelfont=sf,textfont=sf]{subfig} 
\usepackage{textcomp}
\usepackage{stfloats}
\usepackage{url}
\usepackage{verbatim}
\usepackage{graphicx}
\usepackage{color,soul}
\usepackage{hyperref}
\usepackage{booktabs}
\usepackage{amsthm}
\usepackage[numbers]{natbib}
\usepackage{tabularx}
\usepackage{amssymb}
\usepackage[algo2e]{algorithm2e}
\usepackage{comment}
\usepackage{multirow}
\usepackage{dsfont}
\usepackage{multicol}
\newtheorem{theorem}{Theorem}

\newtheorem*{restate}{Theorem}
\begin{document}

\title{Leveraging Programmatically Generated Synthetic Data for Differentially Private Diffusion Training}


\author{
    \IEEEauthorblockN{Yujin Choi\textsuperscript{1}, Jinseong Park\textsuperscript{1}, Junyoung Byun\textsuperscript{2}, Jaewook Lee\textsuperscript{1*}}\\
    \IEEEauthorblockA{\textsuperscript{1}Department of Industrial Engineering, Seoul National University, Seoul, Korea\\
    \textsuperscript{2}Department of Applied Statistics, Chung-Ang University, Seoul, Korea
    }
    \thanks{\{uznhigh, jinseong, jaewook\}@snu.ac.kr, junyoungb@cau.ac.kr}
     \thanks{Jaewook Lee is the corresponding author.}}

{}

\IEEEpubid{}

\maketitle
%
\begin{abstract}
Programmatically generated synthetic data has been used in differential private training for classification to enhance performance without privacy leakage. However, as the synthetic data is generated from a random process, the distribution of real data and the synthetic data are distinguishable and difficult to transfer. Therefore, the model trained with the synthetic data generates unrealistic random images, raising challenges to adapt the synthetic data for generative models.
In this work, we propose \textbf{DP-SynGen}, which leverages programmatically generated synthetic data in diffusion models to address this challenge. 
By exploiting the three stages of diffusion models—coarse, context, and cleaning—we identify stages where synthetic data can be effectively utilized. 
We theoretically and empirically verified that cleaning and coarse stages can be trained without private data, replacing them with synthetic data to reduce the privacy budget. The experimental results show that DP-SynGen improves the quality of generative data by mitigating the negative impact of privacy-induced noise on the generation process. 
\end{abstract}


\section{Introduction}
\IEEEPARstart{A}{dvancements} in deep learning models have led to a rise in critical privacy concerns. Differential privacy (DP) \cite{dwork2006differential} can be applied to mitigate this risk, but it often results in model underfitting. To address this challenge, recent studies have focused on pre-train the model before private training. 
Many of these works have utilized public datasets that can be accessed without concerning privacy. They assume a small ratio of in-distribution public data as public data \cite{park2024distribution} or large-scale out-of-distribution (OOD) public data \cite{de2022unlocking}.
However, these public data may contain sensitive information  \cite{tramer2022position}. Therefore, we should carefully consider its privacy, increasing the importance of training models without relying on public data and fully preserving privacy.

Recently, with advancements in generative models that can generate realistic data, the importance of privacy has grown. 
In particular, several studies have focused on privacy for diffusion models, which have become the de facto standard in recent generative modeling. The diffusion model's small, iterative denoising steps, which are simpler and smoother than single-step generative models, make them converge efficiently and facilitate DP training \cite{dockhorn2022differentially}. 
However, there are two major drawbacks of DP diffusion training: unstable training and large training steps. To address unstable training, \citet{dockhorn2022differentially} proposed the DP diffusion model (DPDM) with noise multiplicity, reusing the same training samples with multiple perturbation levels.

The large training steps arise from the iterative nature of diffusion models. As noted in \citet{de2022unlocking}, DP training has an optimal budget for training iteration. However, since the diffusion model should learn the entire denoising sequence, it requires many training iterations, which increases DP noise per step and degrades performance. Moreover, the sequential denoising process causes noise to accumulate across steps, which negatively impacts generation performance. This performance degradation raises questions about the necessity of private training at every step in diffusion models. As well-optimized pre-trained models can reduce training epochs and promote stable training, we aim to explore pre-training methods that do not raise privacy concerns.

\citet{baradad2021learning} verified that pre-training with programmatically generated synthetic data, which is created from random processes such as automatically-generated fractals, can help to find the better initial weight of models and improve classifier training. As synthetic data are free from privacy leakage, \citet{tang2024differentially} leverages the synthetic data in DP classification, without privacy concerns.
They train a feature extractor with the synthetic data and get state-of-the-art accuracy in DP classification. Moreover, \citet{yu2023vip} leverages programmatically generated synthetic data to DP foundation models.


Compared to pre-training classifiers that can make use of self-supervised learning or feature extractors, finding a better initial model without assuming public data or models is difficult for generative models. 
While public data, even if OOD data, has similar characteristics to private data, its pre-trained features and weights can act as a good initial point for public data. However, programmatically generated synthetic data, derived from random processes, leads to poor initialization during pre-training. 
As diffusion models learn the denoising process, significant discrepancies in the denoising path make pretraining with the synthetic data hard.


To address this, we explore the properties of the diffusion process to verify whether the synthetic data can be effectively used during training. \citet{choi2022perception} identified that the diffusion process involved three stages: coarse, context, and cleaning. Building on this finding, we investigate which stage the synthetic data can be used to reduce the training steps and find a better initial model. 
Following insights from previous research \cite{choi2022perception, choi2024fair}, we found that the context stage of the diffusion process is crucial, even under DP training. Furthermore, we found that private data is unnecessary during the coarse and cleaning stages, allowing us to substitute synthetic data in these phases.

With these findings, we propose \textit{DP-SynGen}, which utilizes the synthetic data to find a better initial model and reduce the number of training steps. 
DP-SynGen reduces the number of training iterations requiring access to private data by leveraging programmatically generated synthetic data in the coarse and cleaning stages, thereby enhancing privacy preservation while maintaining model performance.

Our contributions can be summarized as follows:
\begin{itemize}
\item We alter the private data with programmatically generated synthetic data and reduce the number of training epochs for DP training.
\item We theoretically and empirically identify diffusion stages that synthetic data can help the training and finding a better initial point.
\end{itemize}

\section{Preliminary}
\subsection{Diffusion model}
Diffusion models solve a stochastic differential equation (SDE) to generate data as follows:
\begin{align}\label{eq:diffusion}
    dX_t = f(X_t,t)dt + g(t)dW_t,
\end{align}
where $f(\cdot)$ represents the drift function and $g(\cdot)$ denotes the diffusion coefficient with a Wiener process $dW_t$.
Denoising Diffusion Probabilistic Model (DDPM) \cite{ho2020denoising} set $f(X_t,t) = -\frac{1}{2}\beta_tX_t$ and $g(t) = \sqrt{\beta_t}$, where $\beta_t$ is scheduled hyper-parameter with uniform time stamps $t$. The discrete-time forward stochastic process became 
\begin{align}\label{eq:diffusion_discrete}
X_t = \sqrt{\bar{\alpha}_t}X_0 + \sqrt{1-\bar{\alpha}_t}\eta,
\end{align}
where $\bar{\alpha}_t = \Pi_{s = 1}^t (1-\beta_s)$ and $\eta \sim \mathcal{N}(0, I)$.
As the diffusion model generates data samples from noise, \citet{kingma2021variational} used Signal-to-Noise-Ratio (SNR), defined as $\frac{\bar{\alpha}_t}{1-\bar{\alpha}_t}$, for diffusion sampling. 
Furthermore, 
Elucidating Diffusion Models (EDMs) \cite{karras2022elucidating} redesign diffusion forward process as follows:
\begin{equation}\label{eq:edm}
    X_\sigma = \frac{X_0}{\sqrt{1+\sigma^2}} + \frac{\sigma}{\sqrt{1+\sigma^2}} \eta,
\end{equation}
where $\sigma$ is the noise level,  $\text{SNR}=\frac{1}{\sigma^2}$, and $\bar{\alpha}_\sigma = \frac{1}{1+\sigma^2}$.
EDM uses the $\sigma$-scheduling and thus enables different weights on loss function to training steps $t$.


Building on this, \citet{choi2022perception} categorized the sampling process into three steps: extracting coarse features, producing rich content, and eliminating residual noise. 
High-level features are formed during the coarse stage, while fine details are addressed in the content and clean stages. \citet{choi2024fair} used this property in fair image generation by differing sensitive conditions between the coarse and fine stages.
In terms of efficiency,  \citet{hang2023efficient} proposed a new loss weighting strategy, focusing on high SNR steps.

\subsection{Differential Privacy}
Differential privacy (DP) \cite{dwork2006differential} is formulated as follows: a randomized mechanism $M$ satisfies \((\epsilon, \delta)\)-DP if, for any subset of possible outputs \( S \subseteq \text{Range}(M) \) and for two neighboring datasets $D$ and $D'$, the following condition holds:
\begin{equation}
    \Pr[M(D) \in S] \leq e^{\epsilon} \Pr[M(D') \in S] + \delta.
\end{equation}
In deep learning, DP-SGD \cite{abadi2016deep} is the most conventional method by using gradient clipping and noise addition. To mitigate underfitting in the diffusion training, DPDM \cite{dockhorn2022differentially} introduces noise multiplicity, averaging gradient directions over $k$ different noise levels to reduce the effect of noise in individual gradients.





\section{Proposed Method}

\subsection{Diffusion process involved three stages}\label{sec:motivation}
In this section, we investigate the training stage where synthetic data can be utilized by some toy examples and theoretical analysis. For toy examples, we use the total diffusion step $T = 1000$, and the programmatically generated synthetic data as salt-and-pepper noise, which is well-known for black and white images. 
All the details of toy examples and the detailed proofs for Theorems are in the Appendix.
For simplicity, we refer to programmatically generated synthetic data as synthetic data throughout the remainder of this paper.


\textbf{Importance of context stages.}
We first trained two diffusion models: one on the synthetic images and the other on the MNIST. We then divided the diffusion process into two equal-length parts: one is the context stage for $t\in (250, 750]$ and the other is the stages for $t\in (0, 250]$ and $t\in(750, 1000]$. For each part, a different diffusion model was used - specifically, if the MNIST-trained model was used for the context stage, then the synthetic image-trained model was used for the other stage, and vice versa. The results are shown in the Figure \ref{fig:toy_context}. Even with the same number of diffusion steps, the characteristics of the generated images are different depending on the model used in the context part.
This infers that the context stage is important in sampling images, and cannot be replaced with models trained on different data.

\begin{figure}[ht]
\centering   
    \subfloat[Context stage: Synthetic data, Other stage: MNIST data]{\includegraphics[width=0.233\textwidth]{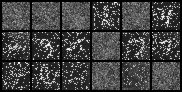} \label{fig:toy_private}}
    \subfloat[Context stage: MNIST data, Other stage: Synthetic data]{\includegraphics[width=0.233\textwidth]{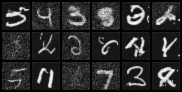} \label{fig:toy_synthetic}}
    \caption{Samples context stage $t \in (250, 750]$ with (a) synthetic data- and (b) private data-trained model. For other stages, the private data- and synthetic data-trained model are used, respectively.}
    \label{fig:toy_context}
\end{figure}

\textbf{Synthetic data can replace coarse stages.}
Next, to verify the usability of synthetic data in the coarse stage, we trained one model using synthetic data and the other using MNIST data for the diffusion step $t\in (\tau, T]$, where $\tau$ serves as the threshold for the coarse stage. Then, both models were trained using MNIST data for the diffusion step $t\in (0, \tau]$. The generated images from each diffusion model ($\tau = 750)$ are illustrated in Figure \ref{fig:toy_coarse}.  



\begin{figure}[ht]
\centering    
    \subfloat[Train with private data]{\includegraphics[width=0.24 \textwidth,  trim=0 88 0 0, clip]{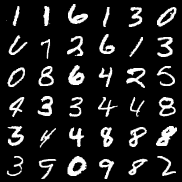} \label{fig:toy_private}}
    \subfloat[Train with synthetic data]{\includegraphics[width=0.24\textwidth,  trim=0 88 0 0, clip]{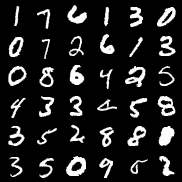} \label{fig:toy_synthetic}}
    \caption{Samples from (a) diffusion model trained with private data for total diffusion process and (b) diffusion model trained with synthetic data for coarse stage ($t > 750$) and private data with other stages ($t \leq 750$).}
    \label{fig:toy_coarse}
\end{figure}

The generated data trained with synthetic data looks similar to the vanilla diffusion model, even though the model is not trained with private data in the coarse stage. We theoretically demonstrate that training the diffusion model with any data can yield similar results as training with private data. In other words, regardless of the initial data distribution, we prove a time step $n$ exists in the diffusion process where the model can be trained within an acceptable margin of error.

\begin{theorem}\label{thm:coarse}
For any two different data distributions $X_0$ and $Y_0$, let $X_t$ and $Y_t$ denote their respective states under the forward diffusion process at time $t$, as defined in Equation \ref{eq:diffusion}.
Then, for any $\nu$ and $\gamma$, we can find $N$ such that for any $n\geq N$, following satisfies:
\begin{align}\label{eq:eps}
    P(\|X_n-Y_n\| >\nu) \leq \gamma.
\end{align}
\end{theorem}

\begin{proof}[sketch of proof]
Use the property that the diffusion process converges probabilistically to normal distribution. 
\end{proof}
Theorem \ref{thm:coarse} demonstrates that for any two data distributions, there exists a point in the diffusion noising process where the distributions become similar in probabilistic. Even if the diffusion model learns $X_t$ for $t>N$ perfectly, the difference between original training and synthetic training can be reduced to within $\nu$, with probability $1-\gamma$. Since this error is negligible in training, replacing the private data with synthetic data in the coarse stage results in an equivalent trained model.
Therefore, by replacing the coarse stage with synthetic data, the performance of the vanilla model can be maintained without any privacy concerns.

\textbf{Synthetic data can replace cleaning stages. }
Finally, for the cleaning stage, we trained one model with synthetic data and the other using MNIST data for the diffusion step $t\in [0, \tau]$, where $\tau$ serves as the threshold for the cleaning stage. Then, we forward the test data into $\tau$, and denoising the data using each trained model.
The generated images from each diffusion model ($\tau = 250)$ are illustrated in Figure \ref{fig:toy_cleaning}.
\begin{figure}[]
\centering    
    \subfloat[Train with private data]{\includegraphics[width=0.24\textwidth,  trim=0 88 0 0, clip]{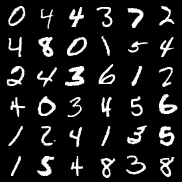} \label{fig:orig_toy_clean}}
    \subfloat[Train with synthetic data]{\includegraphics[width=0.24\textwidth,  trim=0 88 0 0, clip]{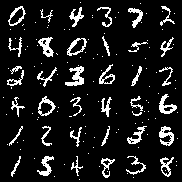} \label{fig:switch_toy_clean}}
    \caption{Samples from (a) diffusion model trained with private data for total diffusion process and (b) diffusion model trained with synthetic data for cleaning stage ($t \leq 250$) and private data with other stages ($t > 250$).}
    \label{fig:toy_cleaning}
\end{figure}

The toy example shows that for cleaning, some characteristics of the synthetic data were reflected. Additional examples with smaller $\tau$ and trained with different types of synthetic data are provided in the Appendix. When $\tau$ was small, the cleaning performance improved significantly, and even the same $\tau$, performance depends on which synthetic data is used.
We provide a Theorem to demonstrate this phenomenon. 

\begin{theorem}\label{thm:cleaning}
    For any two data distributions $X_0$ and $Y_0$ which satisfies $E(\|X_0\|) < \infty$ and $E(\|Y_0\|) < \infty$, and the
    discrete diffusion process is defined as Equation \ref{eq:diffusion_discrete}. 
    For given diffusion process with $\beta_t$-scheduling and for $t \leq \tau$, $\bar{\alpha}_t = \Pi_{s = 1}^t (1-\beta_s)\approx 1$. Then, for any $\nu$, we can find $\gamma$ which depends on the difference between two data distributions and $\bar{\alpha}_t$, such that
    \begin{align*}
        P(\|X_t - X_{0} - (Y_t - Y_{0})\| > \nu) <\gamma.
    \end{align*}
\end{theorem}
\begin{proof}[sketch of proof]
Using Markov's inequality for data distributions, and Chernoff's inequality for $\chi^2$- distribution. 
\end{proof}
Theorem \ref{thm:cleaning} demonstrates that in the cleaning stage, where $\bar{\alpha}_t\approx 1$, the difference between $X_t$ and $X_{0}$ (correspondingly, $Y_t$ and $Y_{0}$) is small, and the discrepancy between two denoising processes remains small, even when using different training data. 
Whereas Theorem \ref{thm:coarse} uses the property of the general diffusion process, the upper bound of Theorem \ref{thm:cleaning} depends on the scheduling of hyper-parameter $\beta_t$ and a norm difference of $X_0$ and $Y_0$. 
Although leveraging synthetic data in the cleaning stage requires a suitable scheduler, this requirement is generally satisfied in the common diffusion processes. Therefore, the synthetic data can be used in the cleaning stage without privacy concerns, by choosing an appropriate $t$.


The findings from toy examples and Theorems suggest the importance of the context stage and demonstrate that using synthetic data in the coarse and cleaning stage can benefit diffusion training, without privacy concerns. 

\subsection{Leveraging synthetic data in DP diffusion training}
In the previous section, we verified utilizing synthetic data in the coarse and cleaning stages can help diffusion training. 
Based on these findings, we propose \textit{DP-SynGen}, a general framework that leverages synthetic data in DP training. Depending on the stage where the synthetic data is utilized, DP-SynGen can make use of two approaches: \textit{DP-SynGen Coarse} and \textit{DP-SynGen Cleaning}. Additionally, to explore the potential of using synthetic data for pre-training, we introduce \textit{DP-SynGen FineTune}, which pre-trains the coarse stage with the synthetic data and trains the total process with private data, i.e., DP-SynGen Coarse with $\tau_2 = \infty$ (will be stated later).


The training steps of DP-SynGen are composed of three stages, defined by which data is used for training; synthetic data, private data, and both. In the first stage, the model is trained on synthetic data. While this can lead to performance degradation due to the use of out-of-distribution data, the absence of private learning requirements allows for potential performance gains. In the other stages, training is conducted solely on private data or on both synthetic and private data together. Since the first stage does not require private learning, we can reduce the number of training epochs for private learning in the other stages. This reduction in training epochs decreases the noise per iteration, leading to more accurate and stable learning.

We adopted the Elucidated Diffusion Model (EDM) \cite{karras2022elucidating}, as in \cite{dockhorn2022differentially}.
Thus, we utilize the standard deviation of the noise $\sigma$, with scheduling as in Equation \ref{eq:edm}. At each training step of EDM, they sample $\ln(\sigma)$ from a normal distribution. Therefore, the threshold $\tau$ were selected for $\ln(\sigma)$. By truncating the normal distribution, the tail is removed, allowing the EDM model to focus on the context phase.

%



To divide the diffusion process into the stages, we sample $\ln(\sigma)$ from a truncated normal distribution, based on thresholds. Specifically, synthetic data is utilized in the cleaning or coarse stage, while private data is used in the remaining stages.
To leverage the synthetic data for DP training, the model is first pre-trained on synthetic data and subsequently trained with private data. 
This two-phase training allows for overlap between stages that use synthetic data and private data, as the three stages of the diffusion process are not strictly divided but gradually mixed. 
This overlap can lead to better performance. Therefore, 
to truncate the normal distribution with a threshold $\tau_1$ for synthetic data training and $\tau_2$ for private training. Figure \ref{fig:process} illustrates our method.
\begin{figure}
    \centering
    \includegraphics[width=0.98\linewidth]{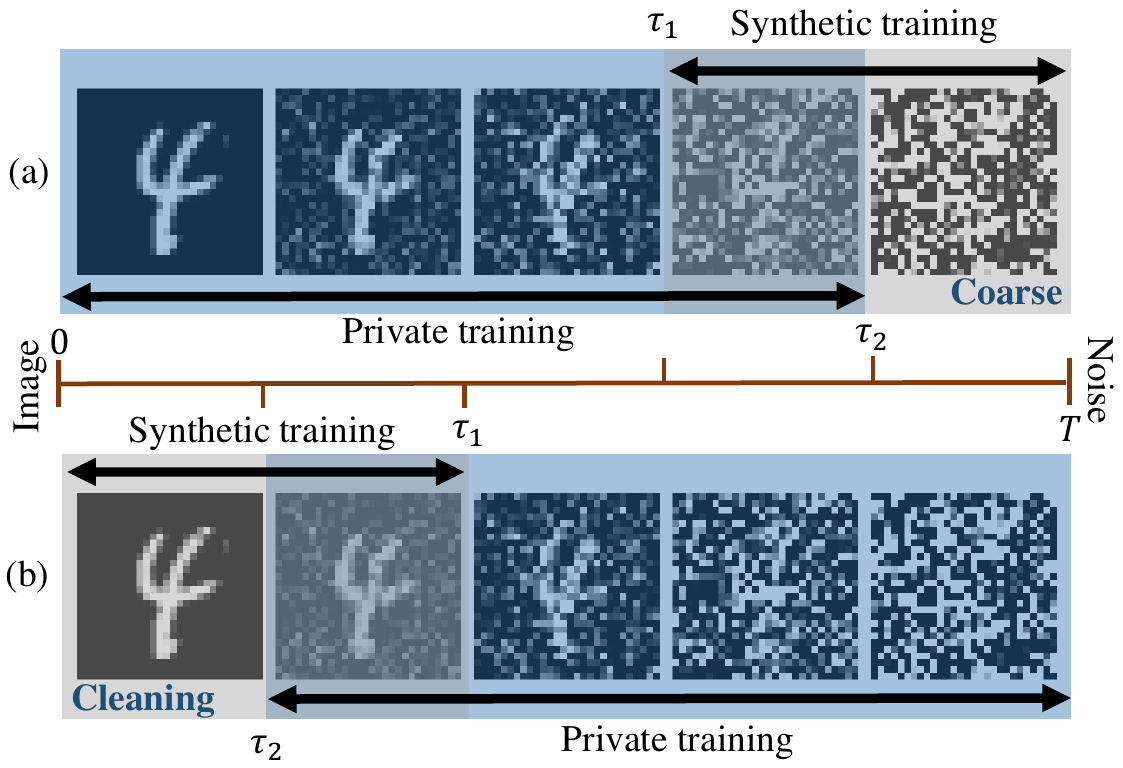}
    \caption{Illustration of (a) DP-SynGen Coarse and (b) DP-SynGen Cleaning, with diffusion process from 0 (Image) to T (Noise). The gray range indicates training with synthetic data, while the blue range indicates training with private data.
    }
    \label{fig:process}
\end{figure}

We describe our method in detail. For DP-SynGen Coarse, which utilizes the synthetic data in the coarse stage, we train the diffusion model with synthetic data when $\ln(\sigma)$ sampled from $\text{TruncatedNormal}(P_{mean}, P^2_{std}|\ln(\sigma) > \tau_1)$. 
After the model converges, the diffusion model is DP-trained with private data when $\ln(\sigma)$ sampled from $\text{TruncatedNormal}(P_{mean}, P^2_{std}|\ln(\sigma) \leq \tau_2)$. 
Generally, we use $\tau_2 \geq \tau_1$ to allow the overlap. 
Otherwise, for DP-SynGen Cleaning, which uses synthetic data in the cleaning stage, we train the diffusion model with the synthetic data when $\ln(\sigma)$ sampled from $\text{TruncatedNormal}(P_{mean}, P^2_{std}|\ln(\sigma) \leq \tau_1)$. 
After the model converges, the diffusion model is DP-trained with private data when $\ln(\sigma)$ sampled from $\text{TruncatedNormal}(P_{mean}, P^2_{std}|\ln(\sigma) > \tau_2)$.

Since it is unnecessary to train where $\ln(\sigma) > \tau_2$ in DP-SynGen Coarse or $\ln(\sigma) \leq \tau_2$ for the DP-SynGen Cleaning, the number of training epochs can be reduced. As \citet{de2022unlocking} mentioned, in DP training, increasing the number of training epochs does not necessarily lead to better performance. 
Our whole framework is illustrated in Algorithm \ref{alg:training} (in  Appendix).

Moreover, not only reducing training epochs but also a good initial point also helps DP training. \citet{tang2024differentially} demonstrated that synthetic data helps find better initial points in classification. 
To verify that training with synthetic data can find the better initial point for the diffusion DP training, we set $\tau_1 = 2$ and $\tau_2 = \infty$ for DP-SynGen Coarse and named this to DP-SynGen FineTune
i.e., we trained the model using the original DPDM model starting from a pre-trained model on synthetic data.

\subsection{Finding optimal thresholds for cleaning and coarse stage}
In this subsection, we find the threshold points $\tau$ based on the SNR and the $\beta_t$-scheduling. Both SNR and $\bar{\alpha}_t = \Pi_{s=1}^t(1-\beta_s)$ reflect the level of noise added during the diffusion process, and since they solely depend on the hyper-parameter, the choice of $\tau_i$ can be calculated without any privacy concerns. 




For DP-SynGen Cleaning, $\beta_t$-scheduling impacts, as we verified in Theorem \ref{thm:cleaning}. Although EDM model does not explicitly defined $\bar{\alpha}_t$, Equation \ref{eq:edm} allows us to compute $\bar{\alpha}_\sigma$, similar to $\bar{\alpha}_t$, based on $\sigma$. Therefore, we choose $\tau_1$ and $\tau_2$ where $\bar{\alpha}_\sigma \approx 1$. The $\bar{\alpha}_\sigma $ varying $\sigma$ is presented in \ref{fig:alpha_bar}.

For DP-SynGen Coarse, we find the coarse stage based on SNR. \citet{choi2022perception} verified that with a small SNR, diffusion learns coarse features. 
Figure \ref{fig:snr_plot} shows the SNR of EDM, where the hyper-parameters follow from \cite{karras2022elucidating}. 
Motivated by \cite{hang2023efficient}, which adjusts training loss weight based on SNR, we inferred that higher SNR is more important for training, but its influence diminishes when it exceeds a certain threshold.
Since the context phase is important in diffusion training, we consider the elbow point as a useful reference for distinguishing between the context and coarse stages. Therefore, we select $\tau_i$ based on elbow point. We hypothesize that the influence of the coarse stage starts around the elbow point and becomes dominant as the curve flattens. We select $\tau_1$ at the elbow point to use synthetic data during the coarse phase and select $\tau_2$ not to use private data at the coarse-dominant stage.

\begin{figure}[!ht]
\centering    
    \subfloat[$\bar{\alpha}_\sigma$ varying $\log(\sigma)$]{\includegraphics[width=0.24\textwidth]{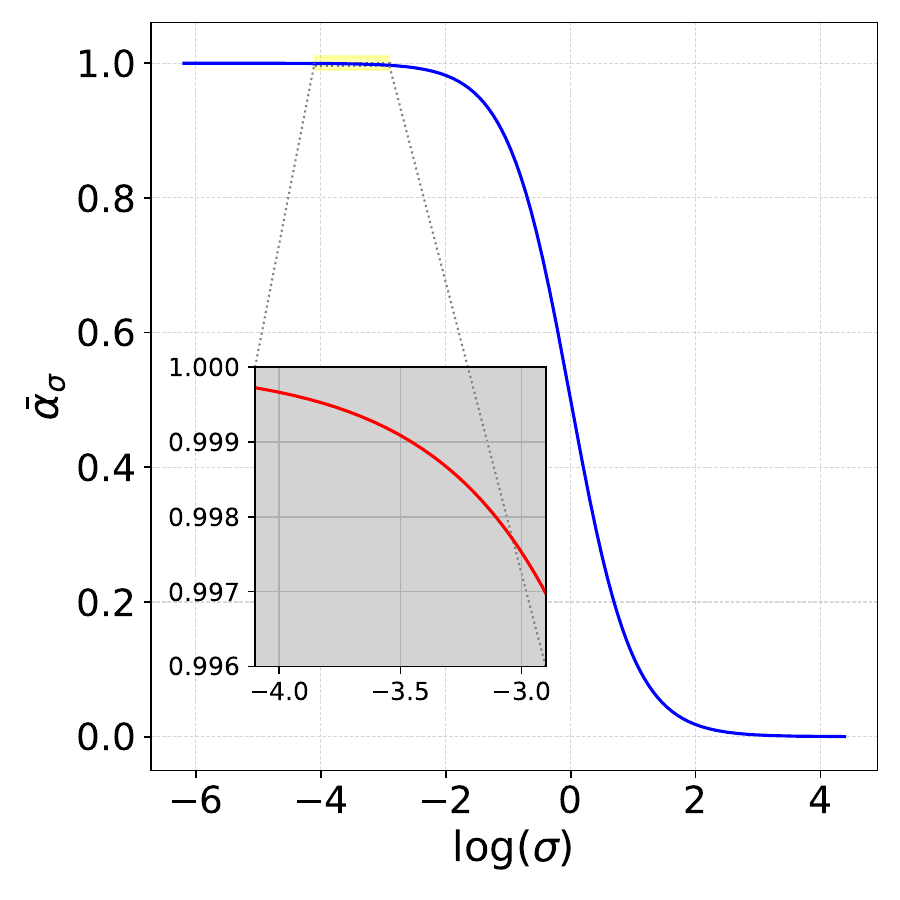} \label{fig:alpha_bar}}
    \subfloat[SNR varying $\sigma$]{\includegraphics[width=0.24\textwidth]{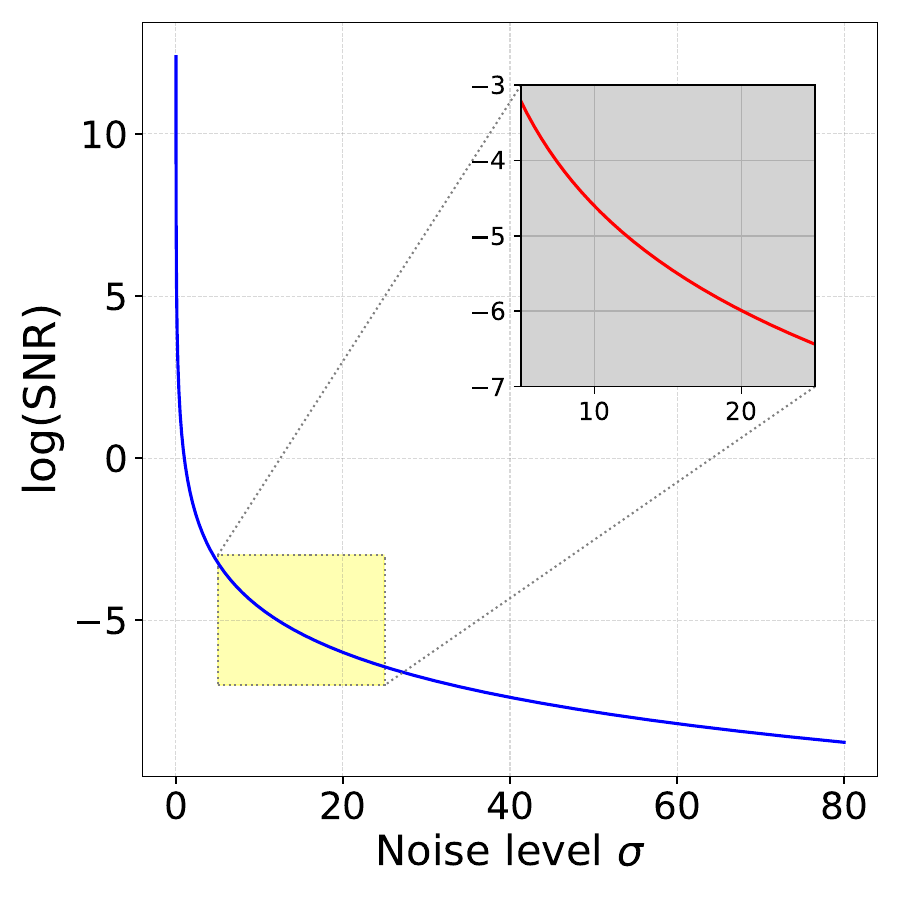} \label{fig:snr_plot}}
    \caption{Visualization of $\bar{\alpha}_\sigma$ and SNR to search the threshold $\tau$}
    \label{fig:snr}
\end{figure}

\section{Experiments}

\begin{table*}[!ht]
\centering
\setlength{\tabcolsep}{8pt} 
\renewcommand{\arraystretch}{0.7}
\begin{tabular}{llcccccccccc}
\toprule
\multirow{3}{*}{Method} & \multirow{3}{*}{DP-$\epsilon$} & \multicolumn{5}{c}{MNIST} & \multicolumn{4}{c}{Fashion MNIST} \\ \cmidrule(lr){3-7} \cmidrule(lr){8-11} 
 &  & \multicolumn{2}{c}{FID} & \multicolumn{3}{c}{Acc (\%)} & FID & \multicolumn{3}{c}{Acc (\%)} \\ \cmidrule(lr){3-4} \cmidrule(lr){5-7} \cmidrule(lr){8-8} \cmidrule(lr){9-11}
 &  & Cond & Uncond & Log Reg & MLP & CNN & Cond & Log Reg & MLP & CNN \\ \midrule
DP-SynGen Coarse & 0.2 & \textbf{141.4} & \textbf{148.4} & 62.9 & 62.5 & 70.0 & 130.5 & \textbf{66.7} & \textbf{66.1} & \textbf{65.8} \\
DP-SynGen Clean & 0.2 & 142.0 & 149.1 & \textbf{64.3} & \textbf{64.2} & \textbf{70.7} & \textbf{124.8} & 63.7 & 64.6 & 64.0 \\
DP-SynGen FineTune & 0.2 & 142.2 & 149.2 & 59.4 & 59.1 & 65.5 & 125.2 & 60.4 & 62.9 & 61.6 \\
DPDM EDM & 0.2 & 142.8 & 163.2 & 59.8 & 58.2 & 67.3 & 125.2 & 60.4 & 62.7 & 61.4 \\ \midrule
DP-SynGen Coarse & 0.5 & 89.6 & \textbf{98.5} & 82.4 & 82.5 & 89.8 & 83.3 & 73.1 & 72.9 & 73.3 \\
DP-SynGen Clean & 0.5 & 91.3 & 100.2 & \textbf{82.7} & \textbf{83.8} & \textbf{90.1} & 83.4 & 73.2 & 73.3 & 74.4 \\
DP-SynGen FineTune & 0.5 & \textbf{84.6} & 102.8 & 81.8 & 83.2 & 89.6 & \textbf{82.9} & \textbf{73.3} & \textbf{74.2} & \textbf{75.5} \\
DPDM EDM & 0.5 & \textbf{84.6} & 102.8 & 80.4 & 82.1 & 90.0 & 83.5 & \textbf{73.3} & \textbf{74.2 }& 74.9 \\ \midrule
DP-SynGen Coarse & 1 & 61.7 & \textbf{64.2} & 85.9 & \textbf{88.1} & \textbf{94.9} & 65.7 & 75.0 & 75.2 & 77.1 \\
DP-SynGen Clean & 1 & 63.0 & 72.1 & 86.1 & \textbf{88.1} & \textbf{94.9} & 67.0 & 75.8 & 76.1 & 77.4 \\
DP-SynGen FineTune & 1 & \textbf{54.8} & \textbf{64.2} & \textbf{86.3} & 88.0 & 94.0 & \textbf{63.2} & \textbf{75.9} & \textbf{76.6} & 78.1 \\
DPDM EDM & 1 & 55.4 & 65.8 & 85.3 & 87.3 & 94.4 & \textbf{63.2} & \textbf{75.9} & 76.2 & \textbf{78.2} \\
DPGANr & 1 & 56.2 & - & - & - & 80.1 & 121.8 & - & - & - \\
DP-HP & 1 & - & - & - & - & 81.5 & 81.5 & 72.3 & - & 72.3 \\ \midrule
DP-SynGen Coarse & 10 & 12.3 & 15.8 & \textbf{89.9} & \textbf{92.4} & \textbf{97.4} & \textbf{21.4} & \textbf{80.5} & 80.6 & \textbf{83.9} \\
DP-SynGen Clean & 10 & 13.7 & 16.0 & \textbf{89.9} & 91.9 & 97.3 & 21.7 & \textbf{80.5} & \textbf{80.9} & 83.7 \\
DP-SynGen FineTune & 10 & 11.9 & 14.0 & 89.6 & 92.0 & 97.1 & 22.1 & 80.2 & \textbf{80.9} & 83.4 \\
DPDM EDM & 10 & \textbf{11.5} & \textbf{13.7} & 89.5 & 92.1 & 97.3 & 22.3 & 79.8 & 80.4 & 82.6 \\
DPGANr & 10 & 13.0 & - & - & - & 95.0 & 56.8 & - & - & 74.8 \\
DP-Sinkhorn & 10 & 48.4 & - & 82.8 & 82.7 & 83.2 & 128.3 & 75.1 & 74.6 & 71.1 \\
G-PATE & 10 & 150.6 & - & - & - & 80.9 & 171.9 & - & - & 69.3 \\
DP-CGAN & 10 & 179.2 & - & 60 & 60 & 63 & 243.8 & 51 & 50 & 46 \\
DataLens & 10 & 173.5 & - & - & - & 80.7 & 167.7 & - & - & 70.6 \\
DP-MERF & 10 & 116.3 & - & 79.4 & 78.3 & 82.1 & 132.6 & 75.5 & 74.5 & 75.4 \\
GS-WGAN & 10 & 61.3 & - & 79 & 79 & 80 & 131.3 & 68 & 65 & 65 \\
\bottomrule
\end{tabular}
\caption{Generation performance measured by FID and CAS, for MNIST and Fashion MNIST. For DPDM, we train and compute ourselves using the EDM model. All other results are taken from the literature.}
\label{tab:main}
\end{table*}
\subsection{Training details}
\textbf{Dataset.}
The experiments were conducted using MNIST and fashion MNIST for the black-and-white image dataset. For colored images, we used the CelebA dataset, which is presented in the Appendix.

For the programmatically generated synthetic images, the type of synthetic data does not affect the diffusion training when it is utilized on the coarse stage, as in Theorem \ref{thm:coarse}. However, in the cleaning stage, the type of synthetic data can affect the quality of the samples, as shown in Figure \ref{fig:toy_cleaning_dead}, Appendix. Therefore, we verify the efficient synthetic data that can be used in the training. 
We utilize `Dead-leaves' images, which have clear boundaries. 

Dead leaves contain various shapes 
with random sizes and colors. 
To increase the contrast of these images, we modified the color selection probabilities. Unlike the original Dead leaves select the color from a uniform distribution over the interval [0,1], we assign a probability of 1/3 to selecting color 0, 1/3 to selecting color 1, and 1/3 to selecting from a uniform distribution over the interval [0,1].

\textbf{Finding Thresholds.}
We set the threshold that makes $\bar{\alpha}_\tau \approx 1$ for DP-SynGen Cleaning. As in the Figure \ref{fig:alpha_bar}, we set $\tau_2 = -3$, where $\bar{\alpha}_{\tau_2} \approx 0.998$ and $\tau_1 = -4$, $\bar{\alpha}_{\tau_1} \approx 0.9997$.
For DP-SynGen Coarse, the elbow point appears near $\ln(\sigma) \approx 2$ and flattens near $\ln(\sigma) \approx 3$, as shown in the Figure \ref{fig:snr_plot}. Therefore, we choose $\tau_1 = 2$, $\tau_2 = 3$.

\textbf{Quality Measure.} We measure the quality of generated data with the Fréchet Inception Distance (FID) \cite{heusel2017gans} score and 
Classification Accuracy Score (CAS) \cite{ravuri2019classification}. 
CAS is the down-stream classifier accuracy which measures the quality of generated data, focusing on the downstream application. 

\textbf{Comparison Methods.}
The experimental results for all comparison methods, except for DPDM \cite{dockhorn2022differentially}, were taken directly from the \citet{dockhorn2022differentially}, which referenced the original sources. Due to the page limitation, details of comparison methods and their references are provided in the Appendix.

\textbf{Hyperparemeters.}
To train DP-SynGen, we randomly labeled the synthetic data. 
The training epochs for DPDM were set to 300, as specified in the original paper, whereas our method was trained for 250 epochs, due to the decreased number of diffusion steps required for DP training. For DP-SynGen FineTune, which does not reduce the diffusion steps, we experimented with both 250 and 300 epochs to determine if the initialization from synthetic data could allow for reduced training while maintaining performance. 
In DPDM and DP-SynGen, although \citet{dockhorn2022differentially} set the multiplicity $k = 32$, we chose $k=16$ since they noted that the performance difference was negligible but computation time was doubled. 

The other hyper-parameters were selected based on the original DPDM paper \cite{dockhorn2022differentially} and its GitHub.
To ensure reproducibility, we will release all experimental code and implementation details at \url{https://github.com/uzn36/DP-SynGen}.
For more details and ablation studies with varying noise multiplicity, refer to the Appendix.

\subsection{Results and analysis}
The results of conditional and unconditional generations are demonstrated in Table \ref{tab:main}. 
The experiments showed that DP-SynGen FineTune generally outperforms DPDM EDM in both FID and CAS, suggesting that synthetic data can provide a better initialization than random initialization. 
However, the performance improvement was modest, showing results comparable to DPDM EDM. On the other hand, DP-SynGen Cleaning and Coarse, with some cases showing performance decreases but others achieving considerable improvements than DPDM EDM.
This indicates the reduction in noise per step has a greater impact on overall performance, which offsets the potential performance degradation caused by the absence of private data in specific stages of training and limited training epochs.

DP-SynGen Cleaning generally performs better in CAS. This is because fine image details have less influence on classifier training compared to high-level features. As shown in the toy example in Figure \ref{fig:toy_cleaning}, replacing the cleaning phase with synthetic data may reduce cleaning performance but still preserve coarse features. 
When the privacy budget is small, the cleaning performance using private data is also limited, reducing the impact of any performance degradation from synthetic data without private learning. Therefore, it can achieve better CAS, specifically in a small privacy budget. 

Moreover, DP-SynGen Coarse achieves better FID scores, particularly at lower privacy budgets. 
As demonstrated in the toy example and the Theorem \ref{thm:coarse}, replacing the private data with the synthetic data in the coarse stage does not change the generated quality, even for nonprivate learning. 
Moreover, as the cleaning stage and the fine image difference have more influence on the FID score and DP-SynGen Coarse reduces the noise per iteration, it can achieve a better FID score.
However, the synthetic data does not have any information about the class condition, it does not give any guidance to the conditional generation. Therefore, the DP-SynGen Coarse improves the FID for unconditional generation more significantly.


\section{Conclusion}
In this paper, we proposed DP-SynGen, a framework that leverages programmatically generated synthetic data for differential private learning. DP-SynGen is categorized into `Cleaning' and `Coarse', based on the diffusion stage where synthetic data is utilized. For each case, we demonstrated its ability to improve DP training, without an additional privacy budget, 
supporting the potential of synthetic data for improving differentially private generative models.
The synthetic data helps mitigate performance loss, particularly when the noise for privacy is large.
We identify a suitable threshold $\tau$, which determines the stages utilizing synthetic data, based on SNR and $\bar{\alpha}_\sigma$, without privacy concerns.



\bibliographystyle{IEEEtranN}
\bibliography{IEEEabrv,ms} 

\begin{thebibliography}{29}
\providecommand{\natexlab}[1]{#1}
\providecommand{\url}[1]{#1}
\csname url@samestyle\endcsname
\providecommand{\newblock}{\relax}
\providecommand{\bibinfo}[2]{#2}
\providecommand{\BIBentrySTDinterwordspacing}{\spaceskip=0pt\relax}
\providecommand{\BIBentryALTinterwordstretchfactor}{4}
\providecommand{\BIBentryALTinterwordspacing}{\spaceskip=\fontdimen2\font plus
\BIBentryALTinterwordstretchfactor\fontdimen3\font minus \fontdimen4\font\relax}
\providecommand{\BIBforeignlanguage}[2]{{%
\expandafter\ifx\csname l@#1\endcsname\relax
\typeout{** WARNING: IEEEtranN.bst: No hyphenation pattern has been}%
\typeout{** loaded for the language `#1'. Using the pattern for}%
\typeout{** the default language instead.}%
\else
\language=\csname l@#1\endcsname
\fi
#2}}
\providecommand{\BIBdecl}{\relax}
\BIBdecl

\bibitem[Dwork(2006)]{dwork2006differential}
C.~Dwork, ``Differential privacy,'' in \emph{International colloquium on automata, languages, and programming}.\hskip 1em plus 0.5em minus 0.4em\relax Springer, 2006, pp. 1--12.

\bibitem[Park et~al.(2024)Park, Choi, and Lee]{park2024distribution}
J.~Park, Y.~Choi, and J.~Lee, ``In-distribution public data synthesis with diffusion models for differentially private image classification,'' in \emph{Proceedings of the IEEE/CVF Conference on Computer Vision and Pattern Recognition}, 2024, pp. 12\,236--12\,246.

\bibitem[De et~al.(2022)De, Berrada, Hayes, Smith, and Balle]{de2022unlocking}
S.~De, L.~Berrada, J.~Hayes, S.~L. Smith, and B.~Balle, ``Unlocking high-accuracy differentially private image classification through scale,'' \emph{arXiv preprint arXiv:2204.13650}, 2022.

\bibitem[Tram{\`e}r et~al.(2022)Tram{\`e}r, Kamath, and Carlini]{tramer2022position}
F.~Tram{\`e}r, G.~Kamath, and N.~Carlini, ``Position: Considerations for differentially private learning with large-scale public pretraining,'' in \emph{Forty-first International Conference on Machine Learning}, 2022.

\bibitem[Dockhorn et~al.(2023)Dockhorn, Cao, Vahdat, and Kreis]{dockhorn2022differentially}
\BIBentryALTinterwordspacing
T.~Dockhorn, T.~Cao, A.~Vahdat, and K.~Kreis, ``{Differentially Private Diffusion Models},'' \emph{Transactions on Machine Learning Research}, 2023. [Online]. Available: \url{https://openreview.net/forum?id=ZPpQk7FJXF}
\BIBentrySTDinterwordspacing

\bibitem[Baradad~Jurjo et~al.(2021)Baradad~Jurjo, Wulff, Wang, Isola, and Torralba]{baradad2021learning}
M.~Baradad~Jurjo, J.~Wulff, T.~Wang, P.~Isola, and A.~Torralba, ``Learning to see by looking at noise,'' \emph{Advances in Neural Information Processing Systems}, vol.~34, pp. 2556--2569, 2021.

\bibitem[Tang et~al.(2024)Tang, Panda, Sehwag, and Mittal]{tang2024differentially}
X.~Tang, A.~Panda, V.~Sehwag, and P.~Mittal, ``Differentially private image classification by learning priors from random processes,'' \emph{Advances in Neural Information Processing Systems}, vol.~36, 2024.

\bibitem[Yu et~al.(2023)Yu, Sanjabi, Ma, Chaudhuri, and Guo]{yu2023vip}
Y.~Yu, M.~Sanjabi, Y.~Ma, K.~Chaudhuri, and C.~Guo, ``Vip: A differentially private foundation model for computer vision,'' \emph{arXiv preprint arXiv:2306.08842}, 2023.

\bibitem[Choi et~al.(2022)Choi, Lee, Shin, Kim, Kim, and Yoon]{choi2022perception}
J.~Choi, J.~Lee, C.~Shin, S.~Kim, H.~Kim, and S.~Yoon, ``Perception prioritized training of diffusion models,'' in \emph{Proceedings of the IEEE/CVF Conference on Computer Vision and Pattern Recognition}, 2022, pp. 11\,472--11\,481.

\bibitem[Choi et~al.(2024)Choi, Park, Kim, Lee, and Park]{choi2024fair}
Y.~Choi, J.~Park, H.~Kim, J.~Lee, and S.~Park, ``Fair sampling in diffusion models through switching mechanism,'' in \emph{Proceedings of the AAAI Conference on Artificial Intelligence}, vol.~38, no.~20, 2024, pp. 21\,995--22\,003.

\bibitem[Ho et~al.(2020)Ho, Jain, and Abbeel]{ho2020denoising}
J.~Ho, A.~Jain, and P.~Abbeel, ``Denoising diffusion probabilistic models,'' \emph{Advances in neural information processing systems}, vol.~33, pp. 6840--6851, 2020.

\bibitem[Kingma et~al.(2021)Kingma, Salimans, Poole, and Ho]{kingma2021variational}
D.~Kingma, T.~Salimans, B.~Poole, and J.~Ho, ``Variational diffusion models,'' \emph{Advances in neural information processing systems}, vol.~34, pp. 21\,696--21\,707, 2021.

\bibitem[Karras et~al.(2022)Karras, Aittala, Aila, and Laine]{karras2022elucidating}
T.~Karras, M.~Aittala, T.~Aila, and S.~Laine, ``Elucidating the design space of diffusion-based generative models,'' \emph{Advances in neural information processing systems}, vol.~35, pp. 26\,565--26\,577, 2022.

\bibitem[Hang et~al.(2023)Hang, Gu, Li, Bao, Chen, Hu, Geng, and Guo]{hang2023efficient}
T.~Hang, S.~Gu, C.~Li, J.~Bao, D.~Chen, H.~Hu, X.~Geng, and B.~Guo, ``Efficient diffusion training via min-snr weighting strategy,'' in \emph{Proceedings of the IEEE/CVF International Conference on Computer Vision}, 2023, pp. 7441--7451.

\bibitem[Abadi et~al.(2016)Abadi, Chu, Goodfellow, McMahan, Mironov, Talwar, and Zhang]{abadi2016deep}
M.~Abadi, A.~Chu, I.~Goodfellow, H.~B. McMahan, I.~Mironov, K.~Talwar, and L.~Zhang, ``Deep learning with differential privacy,'' in \emph{Proceedings of the 2016 ACM SIGSAC conference on computer and communications security}, 2016, pp. 308--318.

\bibitem[Heusel et~al.(2017)Heusel, Ramsauer, Unterthiner, Nessler, and Hochreiter]{heusel2017gans}
M.~Heusel, H.~Ramsauer, T.~Unterthiner, B.~Nessler, and S.~Hochreiter, ``Gans trained by a two time-scale update rule converge to a local nash equilibrium,'' \emph{Advances in neural information processing systems}, vol.~30, 2017.

\bibitem[Ravuri and Vinyals(2019)]{ravuri2019classification}
S.~Ravuri and O.~Vinyals, ``Classification accuracy score for conditional generative models,'' \emph{Advances in neural information processing systems}, vol.~32, 2019.

\bibitem[Song et~al.(2020)Song, Meng, and Ermon]{song2020denoising}
J.~Song, C.~Meng, and S.~Ermon, ``Denoising diffusion implicit models,'' \emph{arXiv preprint arXiv:2010.02502}, 2020.

\bibitem[Yousefpour et~al.(2021)Yousefpour, Shilov, Sablayrolles, Testuggine, Prasad, Malek, Nguyen, Ghosh, Bharadwaj, Zhao, et~al.]{yousefpour2021opacus}
A.~Yousefpour, I.~Shilov, A.~Sablayrolles, D.~Testuggine, K.~Prasad, M.~Malek, J.~Nguyen, S.~Ghosh, A.~Bharadwaj, J.~Zhao \emph{et~al.}, ``Opacus: User-friendly differential privacy library in pytorch,'' \emph{arXiv preprint arXiv:2109.12298}, 2021.

\bibitem[Liew et~al.(2021)Liew, Takahashi, and Ueno]{liew2021pearl}
S.~P. Liew, T.~Takahashi, and M.~Ueno, ``Pearl: Data synthesis via private embeddings and adversarial reconstruction learning,'' \emph{arXiv preprint arXiv:2106.04590}, 2021.

\bibitem[Bie et~al.(2023)Bie, Kamath, and Zhang]{bie2023private}
A.~Bie, G.~Kamath, and G.~Zhang, ``Private gans, revisited,'' \emph{arXiv preprint arXiv:2302.02936}, 2023.

\bibitem[Vinaroz et~al.(2022)Vinaroz, Charusaie, Harder, Adamczewski, and Park]{vinaroz2022hermite}
M.~Vinaroz, M.-A. Charusaie, F.~Harder, K.~Adamczewski, and M.~J. Park, ``Hermite polynomial features for private data generation,'' in \emph{International Conference on Machine Learning}.\hskip 1em plus 0.5em minus 0.4em\relax PMLR, 2022, pp. 22\,300--22\,324.

\bibitem[Cao et~al.(2021)Cao, Bie, Vahdat, Fidler, and Kreis]{cao2021don}
T.~Cao, A.~Bie, A.~Vahdat, S.~Fidler, and K.~Kreis, ``Don’t generate me: Training differentially private generative models with sinkhorn divergence,'' \emph{Advances in Neural Information Processing Systems}, vol.~34, pp. 12\,480--12\,492, 2021.

\bibitem[Long et~al.(2019)Long, Lin, Yang, Gunter, Liu, and Li]{long2019scalable}
Y.~Long, S.~Lin, Z.~Yang, C.~A. Gunter, H.~Liu, and B.~Li, ``Scalable differentially private data generation via private aggregation of teacher ensembles,'' 2019.

\bibitem[Torkzadehmahani et~al.(2019)Torkzadehmahani, Kairouz, and Paten]{torkzadehmahani2019dp}
R.~Torkzadehmahani, P.~Kairouz, and B.~Paten, ``Dp-cgan: Differentially private synthetic data and label generation,'' in \emph{Proceedings of the IEEE/CVF Conference on Computer Vision and Pattern Recognition Workshops}, 2019, pp. 0--0.

\bibitem[Wang et~al.(2021)Wang, Wu, Long, Rimanic, Zhang, and Li]{wang2021datalens}
B.~Wang, F.~Wu, Y.~Long, L.~Rimanic, C.~Zhang, and B.~Li, ``Datalens: Scalable privacy preserving training via gradient compression and aggregation,'' in \emph{Proceedings of the 2021 ACM SIGSAC Conference on Computer and Communications Security}, 2021, pp. 2146--2168.

\bibitem[Harder et~al.(2021)Harder, Adamczewski, and Park]{harder2021dp}
F.~Harder, K.~Adamczewski, and M.~Park, ``Dp-merf: Differentially private mean embeddings with randomfeatures for practical privacy-preserving data generation,'' in \emph{International conference on artificial intelligence and statistics}.\hskip 1em plus 0.5em minus 0.4em\relax PMLR, 2021, pp. 1819--1827.

\bibitem[Chen et~al.(2022)Chen, Yu, Kao, Pang, and Lu]{chen2022dpgen}
J.-W. Chen, C.-M. Yu, C.-C. Kao, T.-W. Pang, and C.-S. Lu, ``Dpgen: Differentially private generative energy-guided network for natural image synthesis,'' in \emph{Proceedings of the IEEE/CVF Conference on Computer Vision and Pattern Recognition}, 2022, pp. 8387--8396.

\bibitem[Jeulin(1997)]{jeulin1997dead}
D.~Jeulin, ``Dead leaves models: from space tessellation to random functions,'' in \emph{Proc. of the Symposium on the Advances in the Theory and Applications of Random Sets}, 1997, pp. 137--156.

\end{thebibliography}

\newpage
\clearpage
\appendix
\section*{Training Details}\label{sec:appen_training}
In this section, we provide the training details of our experiments. For additional training details, refer to the configs folder in  \url{https://github.com/uzn36/DP-SynGen}.
\subsection{Training details for toy examples}
We trained the diffusion model for the toy example using \url{https://github.com/bot66/MNISTDiffusion}. We set the hyper-parameter as their base setting - specifically, training epochs 100, optimizer Adamw with learning rate 0.001, and cycleLR scheduler. The diffusion time steps were 1000, and the model base dimension was 64. 

For programmatically generated synthetic data, we utilized salt-and-pepper noise with probabilistic of the white image with the mean of MNIST data, i.e. 0.13. 

For the denoising process shown in Figure \ref{fig:toy_context}, we scaled the denoised image for each threshold. This is because the scale of the denoising process is different in each model.

\subsection{Training details for Diffusion models}
Our experiments were based on \url{https://github.com/nv-tlabs/DPDM} and set the hyper-parameters as in their config files. Specifically, we utilized the EDM model with an ema rate of 0.999, attention resolution 7 for the MNIST dataset, 
. For training, we set Adam optimizer with learning rate 3e-4 and weight decay 0. For sampling, we used DDIM \cite{song2020denoising} sampler, with stochastic = False. Moreover, we set $T_{max} = 80.0$ and $T_{min} = 0.002$.
Random seeds were set to 0. 


\subsection{Training details for DP training}
The DPDM code was based on the Opacus \cite{yousefpour2021opacus}. 
For DP hyper-parameters, we set
$\delta = 1e-5$, and varying $\epsilon \in \{0.2, 0.5, 1, 10\}$. The max gradient norm $C = 1.0$ and the maximum physical batch size is 8192. 
The noise multiplicity, which is proposed in the DPDM\cite{dockhorn2022differentially}, we set to be $k=16$, to manage the computation bundle. 

\subsection{Training details for CAS}
To measure CAS, we first train classifiers (Logistic Regression, MLP, and CNN) with the generated data and then test the accuracy with the real test data. 
For each model, we set the batch size of 128 and the learning rate of 5e-4. We train all models for 50 epochs, with Adam optimizer. The model architectures are the same as \url{https://github.com/nv-tlabs/DPDM}.

\subsection{Comparison methods}
For other Comparison methods, we get the results from the literature, as reported in  \cite{dockhorn2022differentially}. Note that PEARL \cite{liew2021pearl} was excluded from our comparisons, as the performance reported in the original literature and \cite{dockhorn2022differentially} differed significantly. The other comparison methods included in the main paper are as follows: 
\begin{multicols}{2}
\begin{itemize}
    \item DPGANr \cite{bie2023private}
    \item DP-HP \cite{vinaroz2022hermite}
    \item DP-Sinkhorn \cite{cao2021don}
    \item G-PATE \cite{long2019scalable}
    \item DP-CGAN \cite{torkzadehmahani2019dp}
    \item DataLens \cite{wang2021datalens}
    \item DP-MERF \cite{harder2021dp}
    \item GS-WGAN \cite{chen2022dpgen}
\end{itemize}
\end{multicols}


\section*{Proof of theorems}
In this section, we restate the theorems and provide their proofs.
\begin{restate}[Restatement of Theorem \ref{thm:coarse}]
For any two different data distributions $X_0$ and $Y_0$, let $X_t$ and $Y_t$ denote their respective states under the forward diffusion process at time $t$, as defined in Equation \ref{eq:diffusion}.
Then, for any $\nu$ and $\gamma$, we can find $N$ such that for any $n\geq N$, following satisfies:
\begin{align}\label{eq:eps}
    P(\|X_n-Y_n\| >\nu) \leq \gamma.
\end{align}
\end{restate}
\begin{proof}
    Let $\mathcal{N}$ be the standard normal distribution. As in \cite{}, $X_n\rightarrow \mathcal{N}$ and $Y_n\rightarrow \mathcal{N}$ in probabilistically. Then, for any $\epsilon$ and $\delta$, there exist $N_1$ such that for $n\geq N_1$, $ P(|X_n-\mathcal{N}|>\frac{\epsilon}{2}) \leq \frac{\delta}{2}$ satisfies. Similarly, $N_2$ exists such that for $n\geq N_2$, $ P(|Y_n-\mathcal{N}|>\frac{\epsilon}{2}) \leq \frac{\delta}{2}$ satisfies. Therefore, for $N = \max(N_1, N_2)$ and for any $n \geq N$,
\begin{align*}
P(|X_n-Y_n| >\epsilon) &= P(|X_n-\mathcal{N} + \mathcal{N} - Y_n| >\epsilon)\\
&\leq P(|X_n-\mathcal{N}| + |Y_n - \mathcal{N}| >\epsilon) \\
&\leq P(|X_n-\mathcal{N}|>\frac{\epsilon}{2}) + P(|Y_n - \mathcal{N}| >\frac{\epsilon}{2}) \\
&\leq \delta.
\end{align*}
\end{proof}

\begin{restate}[Restatement of Theorem \ref{thm:cleaning}]
For any two data distributions $X_0$ and $Y_0$ which satisfies $E(\|X_0\|) < \infty$ and $E(\|Y_0\|) < \infty$, and the
    discrete diffusion process is defined as Equation \ref{eq:diffusion_discrete}. 
    For given diffusion process with $\beta_t$-scheduling and for $t \leq \tau$, $\bar{\alpha}_t = \Pi_{s = 1}^t (1-\beta_s)\approx 1 $. Then, for any $\nu$, we can find $\gamma$ which depends on the difference between two data distributions and $\bar{\alpha}_t$, such that
    \begin{align*}
        P(\|X_t - X_{0} - (Y_t - Y_{0})\| > \nu) <\gamma.
    \end{align*}
\end{restate}
\begin{proof}
    Diffusion models follow the equation 
    $X_t = \sqrt{\bar{\alpha}_t}{X_{0}} + \sqrt{1-\bar{\alpha}_t}Z_1$ and $Y_t = \sqrt{\bar{\alpha}_t}{Y_{0}} + \sqrt{1-\bar{\alpha}_t}Z_2$, where $Z_1, Z_2$ follows  standard normal distributions, i.e. $Z_1, Z_2 \sim \mathcal{N}$ and $\bar{\alpha}_t\leq 1$ for every $t$. 
    Then, 
    \begin{align*}
        \|X_t - X_{0} &- (Y_t - Y_{0})\| \\&= \|(\sqrt{\bar{\alpha}_t} - 1)(X_{0}- Y_{0}) +  \sqrt{1-\bar{\alpha}_t}(Z_1-Z_2)\|\\
        &\leq |1-\sqrt{\bar{\alpha}_t} |\|(X_{0}- Y_{0})\| + \sqrt{1-\bar{\alpha}_t}\|(Z_1-Z_2)\|.
    \end{align*}
    Therefore, 
    \begin{align*}
        P(\|&X_t - X_{0} - (Y_t - Y_{0})\| > \nu) \\&\leq P((1-\sqrt{\bar{\alpha}_t})\|(X_{0}- Y_{0})\| + \sqrt{1-\bar{\alpha}_t}\|(Z_1-Z_2)\| > \nu)
        \\&\leq P((1-\sqrt{\bar{\alpha}_t})\|(X_{0}- Y_{0})\|> \frac{\nu}{2}) \\&\quad\quad\quad\quad\quad\quad+ P(\sqrt{1-\bar{\alpha}_t}\|(Z_1-Z_2)\| > \frac{\nu}{2})
    \end{align*}
    Now, we will find each upper bound of $P(\|(X_{0}- Y_{0})\|> \frac{\nu}{2(\sqrt{\bar{\alpha}_t} - 1)})$ and $P(\|(Z_1-Z_2)\| > \frac{\nu}{2\sqrt{1-\bar{\alpha}_t}})$.

    Firstly, let us find the upper bound of $P(\|(Z_1-Z_2)\| > \frac{\nu}{2\sqrt{1-\bar{\alpha}_t}})$.
    Since $(Z_1-Z_2) \sim \mathcal{N}(0, 2I)$, $\|Z_1-Z_2\|^2\sim 2 \chi_d^2$, where $\chi_d^2$ follows a chi-squared distribution with degree $d$ (dimension of $X_0$ and $Y_0$). 
    \begin{align*}
        p(\|Z_1-Z_2\| > \frac{\nu}{2\sqrt{1-\bar{\alpha}_t}}) &= p(\|Z_1-Z_2\|^2 > \frac{\nu^2}{4(1-\bar{\alpha}_t)})
        \\&= p(Z_d > \frac{\nu^2}{8(1-\bar{\alpha}_t)}),
    \end{align*}
    where $Z_d = \|Z_1-Z_2\|^2 / 2 \sim \chi^2_d$. Then  as $\bar{\alpha}_t \approx 1$, applying the Chernoff's bound for chi-squared distribution for large $\frac{\nu}{2\sqrt{1-\bar{\alpha}_t}})$,
    \begin{align*}
        p(Z_d > \frac{\nu^2}{8(1-\bar{\alpha}_t)})\leq \exp(-\frac{\frac{\nu^2}{8(1-\bar{\alpha}_t)}-d}{2} -\frac{d}{2}\ln(\frac{\frac{\nu^2}{8(1-\bar{\alpha}_t)}}{d}))\\
        = \exp(-\frac{\nu^2}{16(1-\bar{\alpha}_t)}+\frac{d}{2} -\frac{d}{2}\ln({\frac{\nu^2}{8d(1-\bar{\alpha}_t)}}))
    \end{align*}

    Next, given condition $E[\|X_0 - Y_0\|] < E[\|X_0\|] + E[\|Y_0\|] < \infty$, we can get the upper bound of $P(\|(X_{0}- Y_{0})\|> \frac{\nu}{2(1-\sqrt{\bar{\alpha}_t})})$ by Markov's inequality:
    \begin{align*}
        P(\|X_{0}- Y_{0}\|> \frac{\nu}{2(1-\sqrt{\bar{\alpha}_t})}) < E(\|X_{0}- Y_{0}\|)\cdot \frac{2(1-\sqrt{\bar{\alpha}_t})}{\nu}
    \end{align*}
    Therefore, 
    \begin{align*}
        P(\|&X_t - X_{0} - (Y_t - Y_{0})\| > \nu) \\& \leq \gamma :=  \frac{2(1-\sqrt{\bar{\alpha}_t})}{\nu} \cdot  E(\|X_{0}- Y_{0}\|)
        \\& \quad\quad\quad
        + \exp(-\frac{\nu^2}{16(1-\bar{\alpha}_t)}+\frac{d}{2} -\frac{d}{2}\ln({\frac{\nu^2}{8d(1-\bar{\alpha}_t)}}))
        \\&
    \end{align*}
    
    Note that if $\bar{\alpha}_t \approx 1$, then $\bar{\alpha}_t \leq \sqrt{\bar{\alpha}_t}\leq1$ and $\bar{\alpha}_t \approx1$ so $(1-{\bar{\alpha}_t})\approx 0$ and $(1-\sqrt{\bar{\alpha}_t}) \approx 0$. Therefore, the condition $\bar{\alpha}_t \approx 1$ makes $\gamma$ is small.
\end{proof}

\section*{Algorithm}
In this section, we provide the training algorithm for DP-SynGen. Based on the stage where synthetic data is used, DP-SynGen is categorized into two types: DP-SynGen Coarse and DP-SynGen Cleaning. Additionally, DP-SynGen FineTune is a special case of DP-SynGen Coarse, where $\tau_2=\infty$. 
The selection of the stage where synthetic data is used is determined by the hyper-parameter \textit{syn\_stage}.
To train the DP-SynGen Coarse (Cleaning), set \textit{syn\_stage} = `coarse' (`cleaning').
The total algorithm is in Algorithm \ref{alg:training}.
\begin{algorithm}[t]\DontPrintSemicolon
\SetAlgoLined
\SetNoFillComment
    \caption{DP-SynGen training}
    \label{alg:training}
    \KwIn {Synthetic data distribution $q_{syn}$, private data distribution $q_{priv}$, EDM hyper-parameters
    $P_{mean}$, $P_{std}$, thresholds $\tau_1, \tau_2$, the choice that which stage synthetic data used \textit{syn\_stage}, training epoch for private learning $N$, initial model $\theta$}
    \KwOut {Trained model parameters $\theta$}
    /* Non-private training with synthetic data*/\\
    \While {converge}{
    $X_0 \sim q_{syn}\\$
    \If{\textit{syn\_stage} = coarse}
    {$\ln(\sigma) \sim \text{TruncatedNormal}(P_{mean}, P^2_{std}|\ln(\sigma) > \tau_1)$}
    \Else
    {$\ln(\sigma) \sim \text{TruncatedNormal}(P_{mean}, P^2_{std}|\ln(\sigma) \leq \tau_1)$}
    Perform Non-private EDM training
    }
    /* DP training with private data*/\\
    \For{epoch = 1 \textbf{to} $N$}{
    $X_0 \sim q_{priv}$ \\
    \If{\textit{syn\_stage} = coarse}{
        $\ln(\sigma) \sim \text{TruncatedNormal}(P_{mean}, P^2_{std} | \ln(\sigma) \leq \tau_2)$ \;
    }
    \Else{
        $\ln(\sigma) \sim \text{TruncatedNormal}(P_{mean}, P^2_{std} | \ln(\sigma) > \tau_2)$ \;
    }
    Perform DPDM training\;
}
\end{algorithm}

\section*{Additional Experiments}
In this section, we provide additional experiments by varying multiplicity and applying our method to other datasets.
\subsection{Additional experiments for toy data}
We provide the generated image for the toy example, varying the synthetic data and $\tau$. The experimental details are the same as the Figure \ref{fig:toy_cleaning}.

First, to validate the effect of the synthetic data with Dead leaves \cite{jeulin1997dead}, in Figure \ref{fig:toy_cleaning_dead}. Although both Figure \ref{fig:toy_cleaning} and Figure \ref{fig:toy_cleaning_dead} used the same threshold $\tau = 250$, the generated quality differed, with denoising using Dead leaved-trained model generated clearer images. Based on this finding, we select Dead leaves for our main experiments.
\begin{figure}[ht]
\centering    
    \subfloat[Train with private data]{\includegraphics[width=0.235\textwidth,  trim=0 88 0 0, clip]{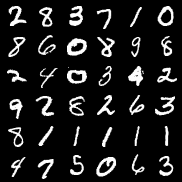}}
    \vspace{0.001cm}
    \subfloat[Train with synthetic data]{\includegraphics[width=0.235\textwidth,  trim=0 88 0 0, clip]{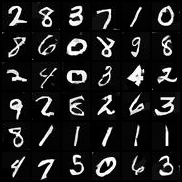} }
    \caption{Samples from (a) diffusion model trained with private data for total diffusion process and (b) diffusion model trained with Dead leaves data for cleaning stage ($t \leq 250$) and private data with other stages ($t > 250$).}
    \label{fig:toy_cleaning_dead}
\end{figure}

Next, for a smaller threshold $\tau = 100$ verify the effect of selecting $\tau$. Figure \ref{fig:toy_cleaning_100} used $\tau = 100$, with the model trained with salt-and-pepper noise image. Compared to Figure \ref{fig:toy_cleaning}, this indicates clearer images, even if they trained on the same synthetic data.
\begin{figure}[ht]
\centering    
    \subfloat[Train with private data]{\includegraphics[width=0.235\textwidth,  trim=0 88 0 0, clip]{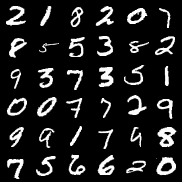}}
    \vspace{0.001cm}
    \subfloat[Train with synthetic data]{\includegraphics[width=0.235\textwidth,  trim=0 88 0 0, clip]{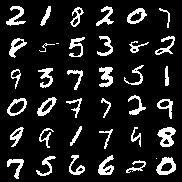} }
    \caption{Samples from (a) diffusion model trained with private data for total diffusion process and (b) diffusion model trained with Salt-and-pepper data for cleaning stage ($t \leq 100$) and private data with other stages ($t > 100$).}
    \label{fig:toy_cleaning_100}
\end{figure}

\subsection{Varying multiplicity}
Noise multiplicity \cite{dockhorn2022differentially} reduces the impact of noise on individual gradients, enabling stable training. To verify this effect, we vary the noise multiplicity. Since noise multiplicity was first introduced in \citet{dockhorn2022differentially}, we are comparing the experimental results of DPDM with our method. 

\begin{table*}[!ht]
\renewcommand{\arraystretch}{1.2} 
\begin{tabular}{lrrrrrrrrrrrrr}
\toprule
\multirow{3}{*}{Method} & \multicolumn{1}{l}{\multirow{3}{*}{DP-$\epsilon$}} & \multicolumn{12}{c}{Multiplicity}                                                                                                                                                                                                                                                                                                     \\ \cline{3-14} 
                        & \multicolumn{1}{l}{}                               & \multicolumn{3}{c}{1}                                                           & \multicolumn{3}{c}{4}                                                           & \multicolumn{3}{c}{8}                                                           & \multicolumn{3}{c}{16}                                                          \\ \cmidrule(lr){3-5} \cmidrule(lr){6-8} \cmidrule(lr){9-11} \cmidrule(lr){12-14}
                        & \multicolumn{1}{l}{}                               & \multicolumn{1}{l}{Log Reg} & \multicolumn{1}{l}{MLP} & \multicolumn{1}{l}{CNN} & \multicolumn{1}{l}{Log Reg} & \multicolumn{1}{l}{MLP} & \multicolumn{1}{l}{CNN} & \multicolumn{1}{l}{Log Reg} & \multicolumn{1}{l}{MLP} & \multicolumn{1}{l}{CNN} & \multicolumn{1}{l}{Log Reg} & \multicolumn{1}{l}{MLP} & \multicolumn{1}{l}{CNN} \\ \hline
Ours Coarse             & 0.2                                                & 22.59                       & 25.15                   & {30.40}          & {51.58}              & 48.77                   & 56.86                   & 53.66                       & 53.08                   & 61.73                   & 62.92                       & 62.53                   & 70.02                   \\
Ours Cleaning           & 0.2                                                & 22.67                       & 25.02                   & {31.58}          & \textbf{52.79}              & \textbf{51.94}          & \textbf{57.43}          & 54.79                       & \textbf{55.56}          & 62.92                   & \textbf{64.27}              & \textbf{64.18}          & \textbf{70.68}          \\
Ours FineTune           & 0.2                                                & \textbf{25.37}              & \textbf{27.42}          & \textbf{31.93}          & {44.87}              & 44.66                   & 55.84                   & \textbf{55.00}              & 54.55                   & 63.12                   & 59.41                       & 59.12                   & 65.48                   \\
DPDM EDM                   & 0.2                                                & {22.45}              & {21.52}          & 26.03                   & 47.48                       & 46.87                   & 55.27                   & 53.29                       & 52.80                   & \textbf{63.56}          & 59.84                       & 58.19                   & 67.29                   \\ \hline
Ours Coarse             & 0.5                                                & {35.85}              & 39.10                   & {52.76}          & {70.63}              & \textbf{74.19}          & 81.82                   & 75.33                       & 76.71                   & 83.32                   & 82.39                       & 82.54                   & 89.82                   \\
Ours Cleaning           & 0.5                                                & \textbf{36.68}              & 36.78                   & \textbf{55.90}          & 71.45                       & 73.82                   & \textbf{81.99}          & 76.11                       & 77.02                   & 83.30                   & \textbf{82.66}              & \textbf{83.75}          & \textbf{90.13}          \\
Ours FineTune           & 0.5                                                & {35.58}              & \textbf{40.50}          & {50.68}          & \textbf{71.51}              & 71.89                   & 80.68                   & \textbf{78.29}              & \textbf{80.49}          & \textbf{88.46}          & 81.80                       & 83.17                   & 89.55                   \\
DPDM EDM                  & 0.5                                                & 35.39                       & {39.86}          & {49.83}          & {70.22}              & 72.13                   & 80.35                   & 77.13                       & 80.00                   & 87.02                   & 80.42                       & 82.14                   & 90.00                   \\ \hline
Ours Coarse             & 1                                                  & {50.13}              & {56.58}          & 67.99                   & {78.90}              & 81.43                   & 88.48                   & 82.73                       & 85.03                   & 91.18                   & 85.94                       & \textbf{88.11}          & 94.89                   \\
Ours Cleaning           & 1                                                  & {51.81}              & 56.21                   & 69.36                   & 79.64                       & 82.20                   & 89.20                   & 83.04                       & 85.65                   & 91.47                   & 86.12                       & 88.08                   & \textbf{94.92}          \\
Ours FineTune           & 1                                                  & \textbf{58.87}              & \textbf{59.51}          & \textbf{70.54}          & \textbf{81.31}              & \textbf{83.34}          & \textbf{90.53}          & 83.20                       & 86.07                   & 92.69                   & \textbf{86.32}              & 87.95                   & 93.95                   \\
DPDM EDM                   & 1                                                  & 53.51                       & {58.93}          & {68.07}          & {80.89}              & 82.23                   & 90.26                   & \textbf{84.10}              & \textbf{86.87}          & \textbf{93.55}          & 85.30                       & 87.30                   & 94.41                   \\ \hline
Ours Coarse             & 10                                                 & \textbf{87.05}              & 90.59                   & \textbf{96.42}          & {88.48}              & 91.75                   & 96.76                   & \textbf{89.76}              & 92.59                   & 97.24                   & \textbf{89.91}              & \textbf{92.41}          & \textbf{97.40}          \\
Ours Cleaning           & 10                                                 & 86.50                       & 90.40                   & {96.29}          & \textbf{89.10}              & 91.75                   & 96.89                   & 89.25                       & 91.82                   & 97.04                   & 89.86                       & 91.90                   & 97.33                   \\
Ours FineTune           & 10                                                 & {86.21}              & \textbf{90.96}          & {96.28}          & 88.72                       & 91.36                   & \textbf{97.01}          & 89.18                       & 92.48                   & \textbf{97.51}          & 89.56                       & 91.98                   & 97.11                   \\
DPDM EDM                   & 10                                                 & 85.86                       & {90.31}          & {95.76}          & {88.48}              & \textbf{91.77}          & 96.94                   & 89.51                       & \textbf{92.84}          & 97.13                   & 89.48                       & 92.10                   & 97.27                \\
\bottomrule
\end{tabular}
\caption{Classifier accuracy score on the MNIST dataset, varying multiplicity and DP-$\epsilon$ values (best scores are highlighted in \textbf{bold}). }\label{tab:mult_cas}
\end{table*}

Table \ref{tab:mult_cas} demonstrates the CAS results for three models with varying multiplicities. Compared to DPDM, the enhancement of DP-SynGen was more significant when the multiplicity was smaller. This is because the impact of DP noise increases as the multiplicity decreases.

\begin{table}[!t]
\renewcommand{\arraystretch}{1.2} 
\begin{tabular}{lrrrrr}
\toprule
\multicolumn{1}{c}{\multirow{2}{*}{Method}} & \multicolumn{1}{c}{\multirow{2}{*}{DP-$\epsilon$}} & \multicolumn{4}{c}{Multiplicity}                                                               \\ \cline{3-6} 
\multicolumn{1}{c}{}                        & \multicolumn{1}{c}{}                               & \multicolumn{1}{c}{1} & \multicolumn{1}{c}{4} & \multicolumn{1}{c}{8} & \multicolumn{1}{c}{16} \\ \hline
DP-SynGen Coarse                                 & 0.2                                                & 180.0                 & 164.1                 & \textbf{143.4}        & \textbf{141.4}         \\
DP-SynGen Cleaning                               & 0.2                                                & 179.9                 & 164.1                 & {144.0}        &  {142.0}         \\
DP-SynGen FineTune                               & 0.2                                                & \textbf{177.2}        &\textbf{158.4}                & {154.6}        & {142.2}         \\
DPDM EDM                                        & 0.2                                                & {180.7}        & {159.3}        & 150.8                 & 142.8                  \\ \hline
DP-SynGen Coarse                                 & 0.5                                                & \textbf{133.7}        & 114.4                 & {101.5}        & {89.6}          \\
DP-SynGen Cleaning                               & 0.5                                                & {134.1}        & 115.3                 & 102.7                 & 91.3                   \\
DP-SynGen FineTune                               & 0.5                                                & {143.8}        & \textbf{112.3}        & \textbf{101.0}        & \textbf{84.6}          \\
DPDM EDM                                        & 0.5                                                & 144.1                 & \textbf{112.3}        & \textbf{101.0}        & \textbf{84.6}          \\ \hline
DP-SynGen Coarse                                 & 1                                                  & \textbf{104.4}        & {77.6}         & 69.1                  & 61.7                   \\
DP-SynGen Cleaning                               & 1                                                  & {105.5}        & 78.5                  & 71.0                  & 63.0                   \\
DP-SynGen FineTune                               & 1                                                  & {112.7}        & 77.5                  & \textbf{62.4}         & \textbf{54.8}          \\
DPDM EDM                                        & 1                                                  & 113.4                 & \textbf{77.4}         & {64.5}         & {55.4}          \\ \hline
DP-SynGen Coarse                                 & 10                                                 & 25.1                  & 16.4                  & \textbf{12.6}         & {12.3}          \\
DP-SynGen Cleaning                               & 10                                                 & 28.7                  & 18.0                  & {14.7}         & 13.7                   \\
DP-SynGen FineTune                               & 10                                                 & {23.8}         & \textbf{15.4}         & {13.5}         & 11.9                   \\
DPDM EDM                                        & 10                                                 & \textbf{23.7}         & {15.5}         & 13.4                  & \textbf{11.5}  \\
\bottomrule
\end{tabular}
\caption{Conditional generation results on the MNIST dataset. FID ($\downarrow$) measured with varying multiplicity and DP-$\epsilon$ values.}\label{tab:mult_fid_cond}
\end{table}

\begin{table}[!t]
\renewcommand{\arraystretch}{1.2} 
\begin{tabular}{lrrrrr}
\toprule
\multicolumn{1}{c}{\multirow{2}{*}{Method}} & \multicolumn{1}{c}{\multirow{2}{*}{DP-$\epsilon$}} & \multicolumn{4}{c}{Multiplicity}                                                               \\ \cline{3-6} 
\multicolumn{1}{c}{}                        & \multicolumn{1}{c}{}                               & \multicolumn{1}{c}{1} & \multicolumn{1}{c}{4} & \multicolumn{1}{c}{8} & \multicolumn{1}{c}{16} \\ \hline
DP-SynGen Coarse                            & 0.2                                                & 183.8                 & 171.6                 & \textbf{152.8}        & \textbf{148.4}         \\
DP-SynGen Cleaning                          & 0.2                                                & 184.0                 & 172.0                 & {153.6}        & {149.1}         \\
DP-SynGen FineTune                          & 0.2                                                & \textbf{180.4}        & \textbf{163.6}        & {158.0}        & {149.2}         \\
DPDM EDM                                        & 0.2                                                & {184.0}        & {171.1}        & 167.6                 & 163.2                  \\ \hline
DP-SynGen Coarse                            & 0.5                                                & \textbf{142.4}        & 120.7                 & {112.1}        & \textbf{98.5}          \\
DP-SynGen Cleaning                          & 0.5                                                & {142.9}        & 121.5                 & 112.8                 & 100.2                  \\
DP-SynGen FineTune                          & 0.5                                                & {153.9}        & \textbf{115.2}        & \textbf{103.7}        & {102.8}         \\
DPDM EDM                                        & 0.5                                                & 155.6                 & {119.6}        & {110.2}        & {102.8}         \\ \hline
DP-SynGen Coarse                            & 1                                                  & {120.9}        & {89.3}         & 75.6                  & \textbf{64.2}          \\
DP-SynGen Cleaning                          & 1                                                  & \textbf{113.9}        & 88.6                  & 82.4                  & 72.1                   \\
DP-SynGen FineTune                          & 1                                                  & {122.1}        & \textbf{82.4}         & \textbf{66.2}         & \textbf{64.2}          \\
DPDM EDM                                        & 1                                                  & 125.3                 & \textbf{82.4}         &{66.3}         & {65.8}          \\ \hline
DP-SynGen Coarse                            & 10                                                 & 33.0                  & 19.0                  & {17.4}         & {15.8}          \\
DP-SynGen Cleaning                          & 10                                                 & 33.4                  & 21.1                  & {17.5}         & 16.0                   \\
DP-SynGen FineTune                          & 10                                                 & \textbf{32.6}         & {19.2}         & {16.3}         & 14.0                   \\
DPDM EDM                                        & 10                                                 & 32.8                  & \textbf{18.9}         & \textbf{15.4}         & \textbf{13.7}         \\
\bottomrule
\end{tabular}
\caption{Unconditional generation results on the MNIST dataset. ($\downarrow$) measured with varying multiplicity and DP-$\epsilon$ values  (best scores are highlighted in \textbf{bold}).}\label{tab:mult_fid_uncond}
\end{table}

The FID results for conditional generation are presented in Table \ref{tab:mult_fid_cond}, and the results for unconditional generation are shown in Table \ref{tab:mult_fid_uncond}.

DP-SynGen tends to achieve greater performance improvements when the privacy budget  $\epsilon$ and noise multiplicity are small. This experimental result suggests that using synthetic becomes more beneficial as the impact of DP noise increases. This is because we trained the model with synthetic data without injecting noise for privacy concerns. As a result, DP-SynGen reduces the injection of DP noise per diffusion step, as it includes phases where DP training is not applied.

\section*{Generated examples}
In this section, we provide the generated examples for each dataset.

\paragraph{MNIST}
The generated example of the MNIST dataset, with $\epsilon \in \{0.2, 0.5, 1, 10\}$. The training details are the same as the above state. 
\begin{figure*}[ht]
\centering    
    \subfloat[MNIST Coarse $\epsilon=0.2$]{\includegraphics[width=0.24\textwidth]{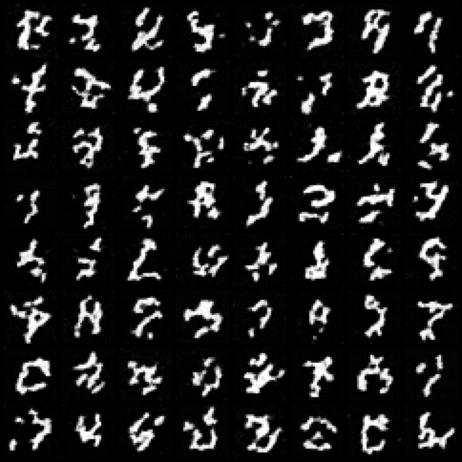}}
    \hspace{0.05cm}
    \subfloat[MNIST Cleaning $\epsilon=0.2$]{\includegraphics[width=0.24\textwidth]{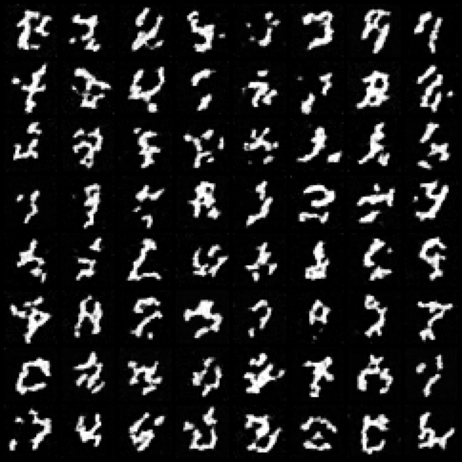}}
    \hspace{0.05cm}
    \subfloat[MNIST Coarse $\epsilon=0.5$]{\includegraphics[width=0.24\textwidth]{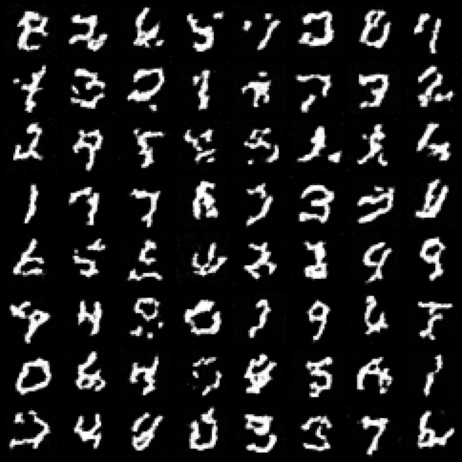}}
    \hspace{0.05cm}
    \subfloat[MNIST Cleaning $\epsilon=0.5$]{\includegraphics[width=0.24\textwidth]{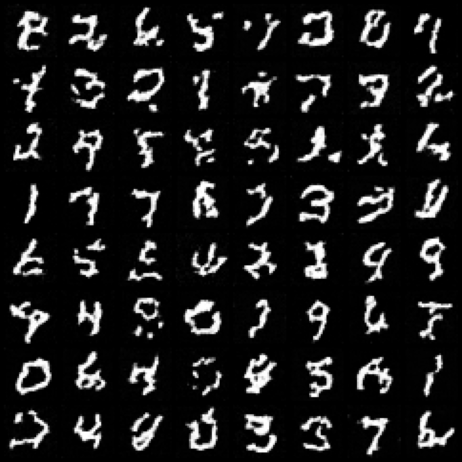}}
    \hspace{0.05cm}
    \subfloat[MNIST Coarse $\epsilon=1$]{\includegraphics[width=0.24\textwidth]{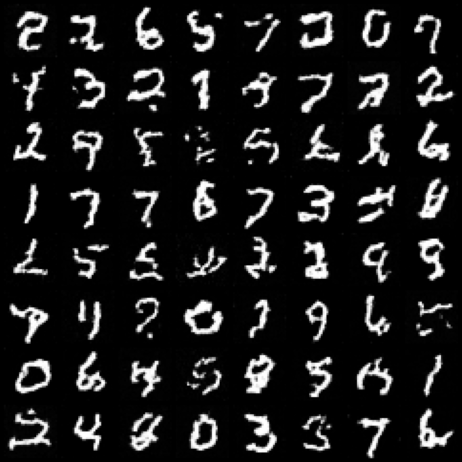}}
    \hspace{0.05cm}
    \subfloat[MNIST Cleaning $\epsilon=1$]{\includegraphics[width=0.24\textwidth]{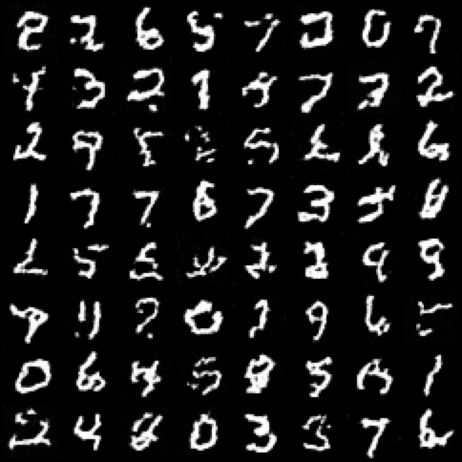}}
    \hspace{0.05cm}
    \subfloat[MNIST Coarse $\epsilon=10$]{\includegraphics[width=0.24\textwidth]{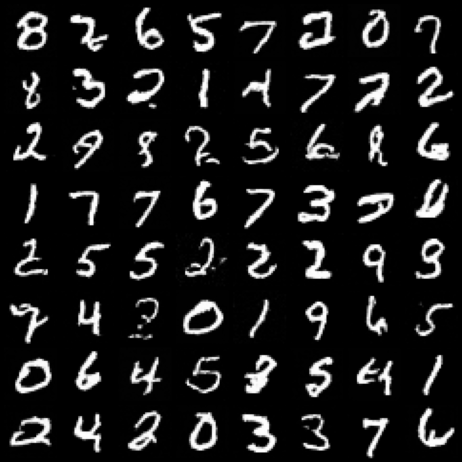}}
    \hspace{0.05cm}
    \subfloat[MNIST Cleaning $\epsilon=10$]{\includegraphics[width=0.24\textwidth]{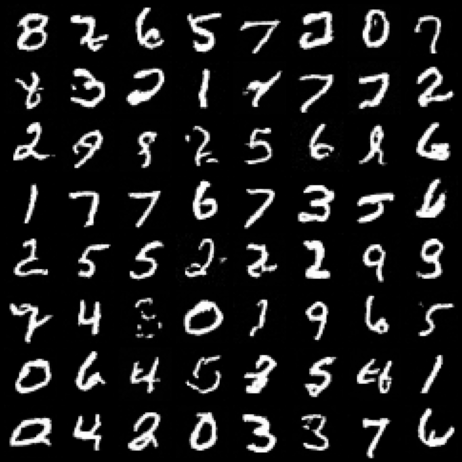}}
    \caption{Generated examples of MNIST dataset.}
    \label{fig:mnist_ours}
\end{figure*}

\paragraph{Fashion MNIST}
The generated example of the Fashion MNIST dataset, with $\epsilon \in \{0.2, 0.5, 1, 10\}$. The training details are the same as the above state. 
\begin{figure*}[ht]
\centering    
    \subfloat[fMNIST Coarse $\epsilon=0.2$]{\includegraphics[width=0.24\textwidth]{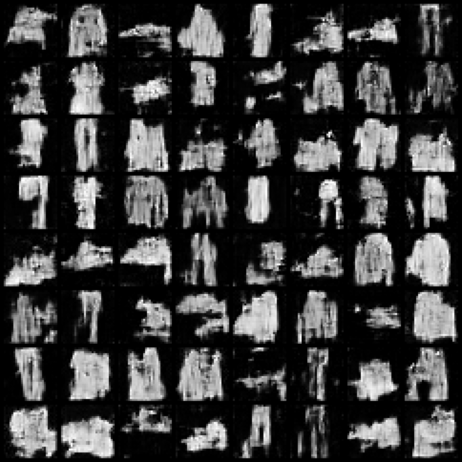}}
    \hspace{0.05cm}
    \subfloat[fMNIST Cleaning $\epsilon=0.2$]{\includegraphics[width=0.24\textwidth]{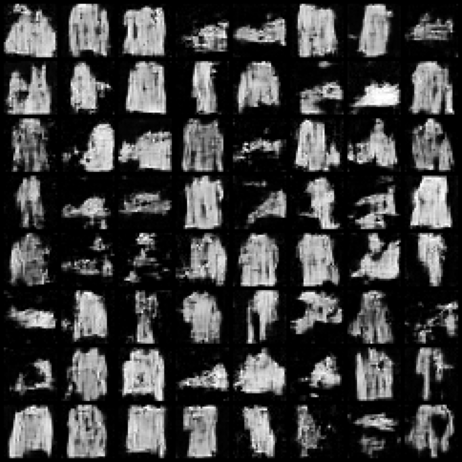}}
    \hspace{0.05cm}
    \subfloat[fMNIST Coarse $\epsilon=0.5$]{\includegraphics[width=0.24\textwidth]{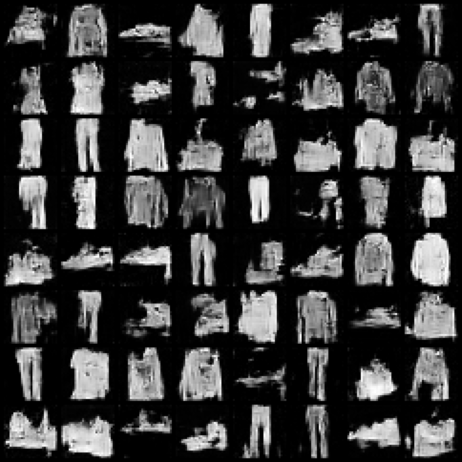}}
    \hspace{0.05cm}
    \subfloat[fMNIST Cleaning $\epsilon=0.5$]{\includegraphics[width=0.24\textwidth]{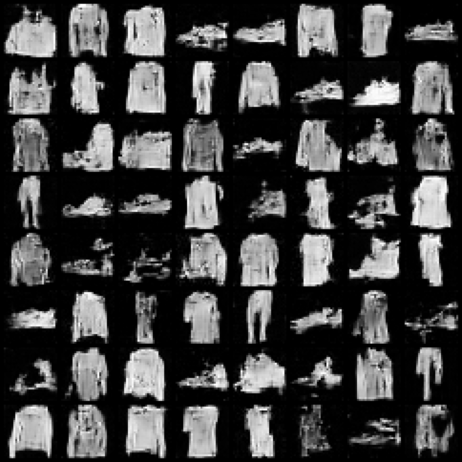}}
    \hspace{0.05cm}
    \subfloat[fMNIST Coarse $\epsilon=1$]{\includegraphics[width=0.24\textwidth]{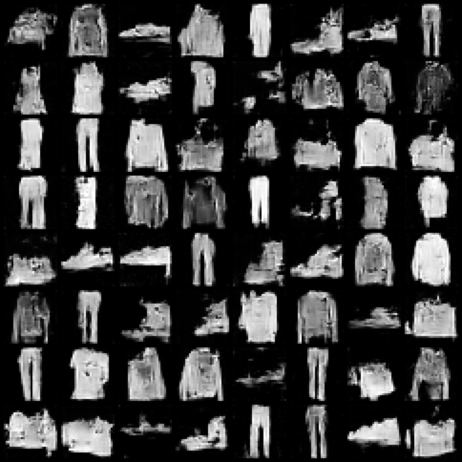}}
    \hspace{0.05cm}
    \subfloat[fMNIST Cleaning $\epsilon=1$]{\includegraphics[width=0.24\textwidth]{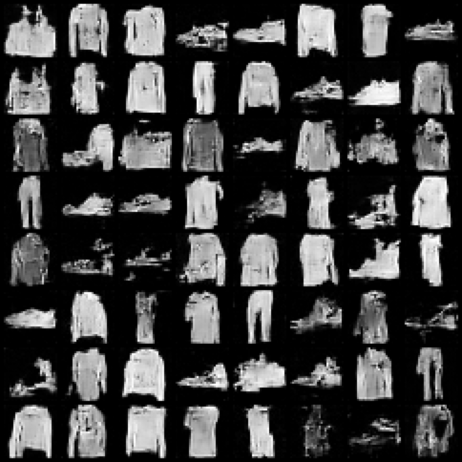}}
    \hspace{0.05cm}
    \subfloat[fMNIST Coarse $\epsilon=10$]{\includegraphics[width=0.24\textwidth]{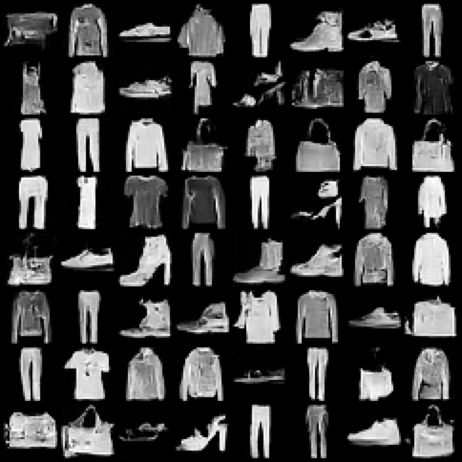}}
    \hspace{0.05cm}
    \subfloat[fMNIST Cleaning $\epsilon=10$]{\includegraphics[width=0.24\textwidth]{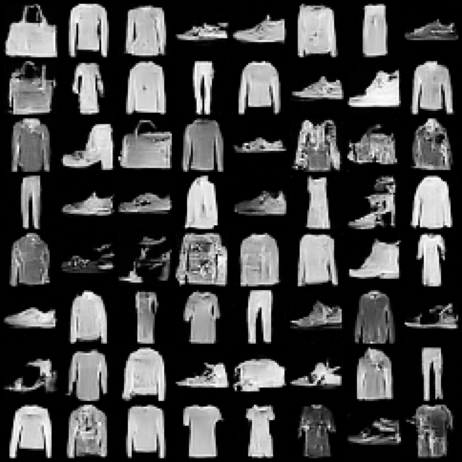}}
    \caption{Generated examples of Fashion MNIST dataset.}
    \label{fig:fmnist_ours}
\end{figure*}

\paragraph{CelebA}
To validate the potential to adapt DP-SynGen in the color image, we conduct experiments for the CelebA dataset.
For CelebA data, we used the Dead-leaves data for programmatically generated synthetic data. Noise multiplicity for DP training is set to be $k=16$, and training epochs were 100. 
\begin{figure}
    \centering
    \includegraphics[width=0.8\linewidth]{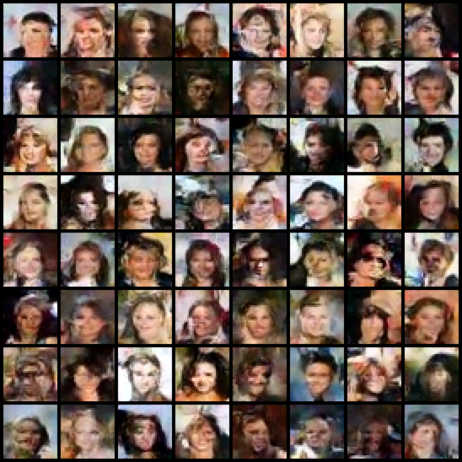}
    \caption{Generated example of CelebA dataset with DP-SynGen Cleaning method, with $\epsilon = 5$}
    \label{fig:celeba_clean}
\end{figure}

\begin{figure}
    \centering
    \includegraphics[width=0.8\linewidth]{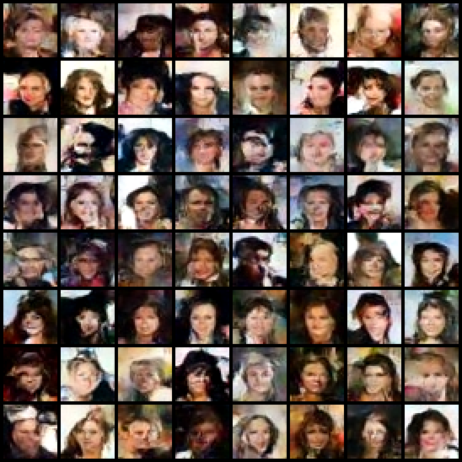}
    \caption{Generated example of CelebA dataset, with DP-SynGen Coarse method, $\epsilon = 5$}
    \label{fig:celeba_coarse}
\end{figure}







\end{document}